\documentclass[letterpaper]{article}
\usepackage[preprint]{aaai2027}
\usepackage[hyphens]{url}
\usepackage{graphicx}
\usepackage{natbib}
\usepackage{caption}
\usepackage{amsmath,amssymb,amsfonts,mathtools}
\usepackage{booktabs}
\usepackage{array}
\usepackage{multirow}
\usepackage{placeins}
\newcommand{\best}[1]{\textbf{\boldmath #1}}
\usepackage{amsthm}
\newtheorem{proposition}{Proposition}
\newtheorem{assumption}{Assumption}

\title{CORA-Diff: Confidence-Oriented Residual Acceptance for Efficient Diffusion Language Model Inference}
\author{Yifan Wu, Yufeng Zhang\corresponding, Kenli Li}
\affiliations{
College of Computer Science, Hunan University\\
Changsha 410082, China\\
\{wuyf,yufengzhang,lkl\}@hnu.edu.cn
}

\begin{document}
\maketitle
\begin{abstract}
Diffusion language models (DLMs) update many tokens in parallel, yet
practical decoders often use a fixed denoising horizon. Many predictions
stabilize early, but blockwise decoding continues until the remaining
positions are resolved, leaving repeated dense forward passes. Existing
accelerators reduce this redundancy but often rely on learned filters,
modified scores, dependency models, or cache-specific mechanisms. We
instead ask whether native trajectory signals alone can identify residual
positions likely to match the deterministic dense endpoint. We propose
\textbf{CORA-Diff}, a training-free method that preserves the original
transfer rule and applies confidence-and-persistence gating only to
positions that rule leaves unresolved. Accepted tokens remain visible as
context, and the block terminates once all positions are resolved. This
requires no backbone change, learned acceptance model, or logit
modification. We analyze reliability through agreement with the
fixed-horizon dense endpoint. Our theory explains why high-confidence,
persistent predictions are more likely to match this endpoint, and
paired post-intervention trajectories provide direct empirical support.
We select one operating point on a separate GSM8K calibration subset and
freeze it for all subsequent evaluations.
Under a matched Learn2PD-style LLaDA protocol, CORA-Diff has the lowest
measured runtime in all eight task--length settings. Task scores match or
exceed dense decoding in five settings, and the largest observed drop is
1.22 points. Its incremental speedups over EOS-aware dense decoding are
$2.70\times$ and $3.32\times$ on GSM8K and HumanEval. It
also reaches $13.14\times$ under the fixed-horizon 1024/1024
mechanism-isolation protocol and transfers to Dream without retuning at
$3.18\times$--$3.53\times$. These results demonstrate that native
confidence and persistence are sufficient in practice for reliable
residual acceptance, allowing a simple training-free gate to remove
substantial repeated denoising computation while preserving task quality.
\end{abstract}

\begin{links}
\link{Code}{https://github.com/wyffffff/cora-diff-llada}
\end{links}

\section{Introduction}
\label{sec:introduction}

Autoregressive language models generate tokens sequentially, whereas
diffusion language models (DLMs) can update many positions in
parallel~\citep{austin2021structured,lou2024sedd,shi2024simplified,sahoo2024mdlm}.
This parallelism is promising, but it does not guarantee fast inference.
Many practical DLM decoders use a fixed denoising horizon and execute one
dense Transformer pass per step~\citep{nie2025llada,ye2025dream}. In
blockwise variants, this horizon is assigned to each generation
block~\citep{nie2025llada,arriola2025blockdiffusion}. These passes may
continue after many predictions have stabilized. Because every executed
pass remains dense, accepting individual tokens reduces latency only when
the original transfer rule and early acceptance jointly resolve the active
block. The central challenge is therefore to determine when positions left
unresolved by the original rule are reliable enough to accept, so that the
active block can finish and skip its remaining dense passes.

Prior work confirms that fixed-horizon decoding can waste computation.
Prophet observes \emph{early answer convergence}, and Learn2PD shows
that some tokens match the final output well before decoding
ends~\citep{li2025prophet,bao2025learn2pd}. Existing accelerators use
learned acceptance~\citep{bao2025learn2pd,chen2025dparallel}, logit
modification~\citep{wang2025creditdecoding}, dependency-aware or adaptive
schedules~\citep{luo2026dawn,kim2026dapd,ringel2026demask,wei2025accelerating},
cache reuse or focused computation~\citep{wu2025fastdllm,liu2025dllmcache,hu2025freecache,liang2026focus},
and speculative checking~\citep{pan2026blockspec}. These routes are
effective, but they often require learned components, modified decision
scores, explicit dependency models, or specialized system support. We ask
a complementary question: does the native denoising trajectory already
contain enough evidence to accept unresolved positions without extra
training, backbone changes, or logit modification?

Early acceptance faces two immediate failure modes: the current
prediction may be weakly supported, or it may change as denoising
evolves. These failure modes suggest two complementary signals already
available in the native denoising trajectory: top-1 confidence measures
instantaneous support, while cross-step persistence measures trajectory
stability. Both are computed from outputs already produced at each
executed step, requiring no auxiliary model or extra forward pass. We
test this hypothesis by comparing residual-position predictions with the
fixed-horizon dense endpoint and stratifying disagreement by confidence.

\begin{figure}[t]
\centering
\includegraphics[width=0.98\columnwidth]{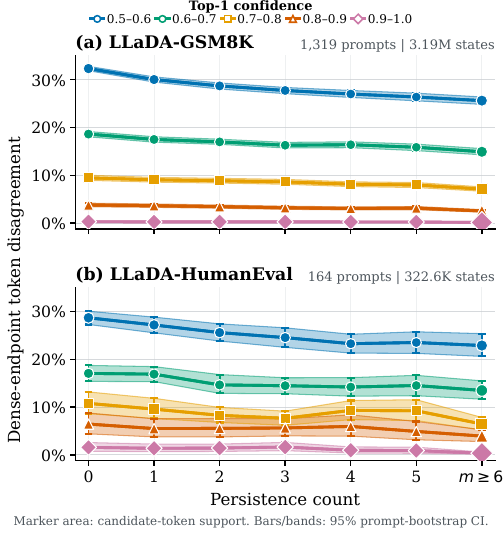}
\caption{
Dense-trace diagnostic for early acceptance of unresolved tokens.
On dense LLaDA traces from all 1,319 GSM8K and 164 HumanEval test prompts, token disagreement with the full-horizon dense endpoint generally decreases with persistence within fixed top-1 confidence ranges.
Bands and error bars show 95\% confidence intervals obtained by resampling prompts, marker area indicates the number of candidate-token states, and $m\ge6$ is pooled.
The overall decline shows that persistence provides reliability information beyond confidence alone.
}
\label{fig:persistence_risk}
\end{figure}

Figure~\ref{fig:persistence_risk} shows that, within fixed confidence
ranges, persistent predictions are less likely to disagree with the
dense endpoint on unperturbed LLaDA traces. Prompt-level bootstrap
intervals account for dependence among states from the same prompt.
Although individual bins contain local reversals, both tasks show the
same overall pattern: persistence adds reliability information beyond
confidence. Because committing tokens early can change later
predictions, we also test on actual CORA-Diff runs whether persistent
predictions remain more likely to match the paired dense endpoint.

Motivated by this observation, we propose \textbf{CORA-Diff}
(\textbf{Co}nfidence-Oriented \textbf{R}esidual \textbf{A}cceptance for
Diffusion Language Models), a training-free residual acceptance decoder.
CORA-Diff preserves the original transfer rule and applies its gate only
to positions that rule leaves unresolved. A position is accepted only
when its top-1 prediction is both confident and persistent. Positions
fixed by the original rule or accepted by CORA-Diff remain visible as
context. Once the two rules jointly resolve the block, decoding stops
and the remaining dense passes are skipped. Thus, CORA-Diff reduces the
number of executed steps without an auxiliary verifier, backbone
changes, or logit modification.

We evaluate reliability by comparing early accepted tokens with the
fixed-horizon dense endpoint. Our analysis explains why high confidence
and persistent predictions favor endpoint agreement, and paired
post-intervention trajectories test this relationship on actual
CORA-Diff states. We select thresholds on an independent calibration
set, characterize this selection with a finite-sample bound, and then
measure task quality and block-level computation savings on held-out
benchmarks.

We select $(\delta_p,m)=(0.65,1)$ on a separate 1,000-prompt GSM8K
calibration subset and freeze it for all reported test and
cross-backbone transfer experiments.
We evaluate LLaDA-8B-Instruct on GSM8K, MATH, HumanEval, and MBPP, and
test transfer to Dream. CORA-Diff has the lowest measured runtime in all
eight matched LLaDA settings; task scores match or exceed dense decoding
in five settings, and the largest decrease is 1.22 points.
Under a shared EOS-aware controller implemented with end-of-text
prediction (EoTP), decoding terminates only after a committed
end-of-sequence (EOS) token. Its incremental deployment speedups are
$2.70\times$ and $3.32\times$ over EOS-aware dense decoding on GSM8K
and HumanEval. The fixed-horizon 1024/1024 mechanism-isolation protocol
gives the larger $13.14\times$ result.
With the same thresholds and no retuning, CORA-Diff reaches
$3.18\times$--$3.53\times$ speedup on Dream, with task-score changes of
at most 0.006.

Our contributions are:
\begin{itemize}
    \item We propose \mbox{\textbf{CORA-Diff}}, a training-free residual
    acceptance method without logit modification.
    It applies confidence-and-persistence gating only to positions left
    unresolved by the original transfer rule.

    \item We formalize early acceptance through dense-endpoint
    disagreement and give a margin-based sufficient condition for
    accepted-token agreement. Under an explicit at-risk
    persistence--drift separation assumption, we derive a conditional
    drift-tail bound and analyze the actual calibration selector.

    \item We show that this lightweight rule enables block-level early
    termination.
    It achieves $2.70\times$--$3.32\times$ incremental speedups over
    EOS-aware dense decoding, reaches $13.14\times$ in fixed-horizon
    mechanism isolation, and
    transfers to Dream without retuning.
\end{itemize}

\section{Related Work}
\label{sec:related_work}

\paragraph{Diffusion language models.}
Diffusion language models generate text by iteratively denoising or
unmasking corrupted sequences. Prior work studies discrete corruption,
continuous diffusion over embeddings, and masked or score-based
objectives~\citep{austin2021structured,campbell2022continuous,li2022diffusionlm,lou2024sedd,shi2024simplified,sahoo2024mdlm,zheng2024masked,ou2025absorbing}.
Large-scale examples include LLaDA, DiffuLLaMA, and
Dream~\citep{nie2025llada,gong2025diffullama,ye2025dream}.
Block Diffusion combines autoregressive dependencies across blocks with
diffusion refinement within each block~\citep{arriola2025blockdiffusion}.
CORA-Diff leaves these modeling choices and the pretrained backbone
unchanged and targets repeated work during inference.

\paragraph{Acceptance and step reduction.}
DLM accelerators use different reliability signals and therefore stop
at different operating points~\citep{kang2025parallelbench}.
The four baselines in our matched comparison span the closest routes:
Prophet uses early answer convergence, KLASS uses distributional
stability, DAPD uses dependency-aware decoding, and Learn2PD trains a
lightweight filter for endpoint agreement~\citep{li2025prophet,kim2025klass,kim2026dapd,bao2025learn2pd}.
Other methods adapt block sizes, confidence schedules, or token
clusters~\citep{lu2025adablock,mohamed2025sched,luo2026dico}.
Learned policies and trace-based rules further use model training,
self-distillation, or decoding histories~\citep{chen2026dmax,chen2025dparallel,wang2025creditdecoding,sun2026tracelock,li2026tspd}.
Dependency-aware methods explicitly model interactions among masked
positions~\citep{luo2026dawn,ringel2026demask,sahin2026adas}.
CORA-Diff targets block completion rather than token acceptance alone:
it clears the last unresolved positions so that the block can skip
remaining dense passes.

\paragraph{Per-step and system acceleration.}
An orthogonal line reduces the cost of each executed step through KV
reuse, focused computation, speculative verification, or expert
offloading~\citep{wu2025fastdllm,liu2025dllmcache,hu2025freecache,liang2026focus,pan2026blockspec,chen2026tide}.
WeDLM changes the attention structure and reorders generation to enable
prefix KV caching and streaming decoding~\citep{liu2026wedlm}.
These methods reduce per-step or system cost, whereas CORA-Diff reduces
the number of executed dense steps; the cache experiment tests their
compatibility. This distinction also separates CORA-Diff from broader
adaptive-computation methods that change layer, width, or token-level
execution~\citep{goyal2020powerbert,hou2020dynabert,elhoushi2024layerskip,raposo2024mod,devvrit2024matformer}.

\section{Method}
\label{sec:method}

We propose \textbf{CORA-Diff} (\textbf{Co}nfidence-Oriented \textbf{R}esidual \textbf{A}cceptance for Diffusion Language Models), a training-free DLM decoder.
It freezes the backbone and retains dense Transformer execution at every step that is run.
It reduces the number of such steps by accepting unresolved positions whose top-1 predictions are both confident and persistent.
Figure~\ref{fig:cora_overview} summarizes the procedure.
Accepted tokens remain visible as context.
When no unresolved position remains, CORA-Diff terminates the block and skips the remaining forward passes.

\begin{figure*}[t]
\centering
\includegraphics[width=0.95\textwidth]{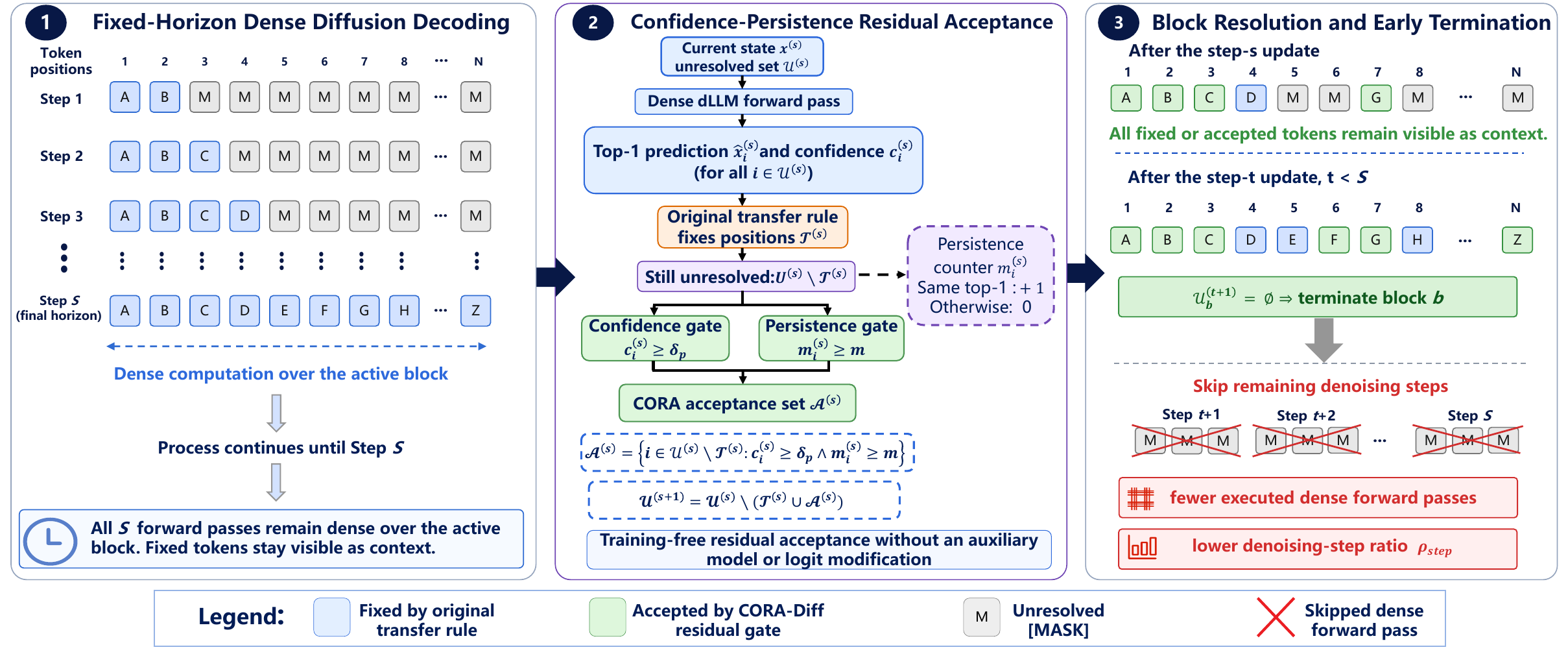}
\caption{
Overview of CORA-Diff.
Standard decoding executes a fixed denoising horizon for each block.
At each step, the original transfer rule first fixes a subset of positions.
CORA-Diff applies its residual gate only to the remaining positions and accepts a token when it satisfies both $c_i^{(s)} \ge \delta_p$ and $m_i^{(s)} \ge m$.
Tokens fixed by either rule remain visible as context.
Once no unresolved position remains, the block terminates and skips the remaining dense forward passes.
}
\label{fig:cora_overview}
\end{figure*}
\subsection{Diffusion Decoding Setup}
\label{sec:diffusion_setup}

Let $\mathbf{x}^{(s)}=(x_1^{(s)},\dots,x_n^{(s)})$ be the partially denoised sequence at step $s$.
Let $\mathcal{U}^{(s)}$ be the set of unresolved positions.
For each unresolved token $i\in\mathcal{U}^{(s)}$, the diffusion language model produces logits $\boldsymbol{\ell}_i^{(s)}$ and token distribution $\mathbf{p}_i^{(s)}=\operatorname{softmax}(\boldsymbol{\ell}_i^{(s)})$.
The current prediction and its confidence are
\begin{equation}
\hat{x}_i^{(s)}=\arg\max_v p_{i,v}^{(s)},
\qquad
c_i^{(s)}=\max_v p_{i,v}^{(s)} .
\end{equation}

In the blockwise masked-diffusion decoding setup used by LLaDA,
generation tokens are partitioned into blocks and decoded one block at a time.
Each block has a fixed denoising horizon.
At each step, the model predicts all currently masked positions, after which
the original transfer rule fixes a subset by confidence or remasking.
This conservative schedule can keep refining already-stable predictions.
CORA-Diff adds training-free early acceptance on top of the original schedule.

\subsection{Confidence-Oriented Residual Acceptance}
\label{sec:residual_acceptance}

At temperature zero, the trajectory and full-horizon endpoint are
deterministic for each prompt. Let $Q\sim\mathcal D$ and let
$Y_Q^\star$ be the output of the original dense decoder. At the prompt
level, probabilities are induced by $Q\sim\mathcal D$; the state-level
analysis in Section~\ref{sec:theoretical_motivation} additionally draws
a token--step state from a fixed offline analysis measure. We evaluate
early decisions by population disagreement with $Y_Q^\star$, which
measures deviation from dense decoding rather than task correctness.

CORA-Diff uses top-1 confidence and cross-step prediction persistence,
both available from the current decoding trajectory.

For each unresolved token, we maintain a persistence counter:
\begin{equation}
\label{eq:persistence_counter}
m_i^{(s)}
=
\begin{cases}
m_i^{(s-1)}+1, & \text{if } \hat{x}_i^{(s)}=\hat{x}_i^{(s-1)},\\
0, & \text{otherwise}.
\end{cases}
\end{equation}
A larger $m_i^{(s)}$ means that the same prediction has persisted for more consecutive denoising steps.
At a position's first executed step the counter is zero. Thus $m=0$
is confidence-only, whereas $m=1$ requires the same top-1 prediction
on two consecutive executed steps; skipped steps do not update it.

CORA-Diff accepts token $i$ at step $s$ if it is both confident and persistent:
\begin{equation}
\label{eq:cora_acceptance}
a_i^{(s)}
=
\mathbb{I}
\left[
c_i^{(s)}\ge\delta_p
\;\land\;
m_i^{(s)}\ge m
\right],
\end{equation}
where $\delta_p$ is the confidence threshold and $m$ is the required persistence window.
Let $\mathcal{T}^{(s)}\subseteq\mathcal{U}^{(s)}$ be the positions fixed by the original transfer rule at step $s$.
The residual gate is applied only to the remaining positions; let $\mathcal{A}^{(s)}$ denote the positions it accepts.
Each token $i\in\mathcal{A}^{(s)}$ is fixed to its current prediction, and the unresolved set is updated by both rules:
\begin{equation}
\begin{aligned}
x_i^{(s+1)}&\leftarrow \hat{x}_i^{(s)}
&& (i\in\mathcal{A}^{(s)}),\\
\mathcal{U}^{(s+1)}
&=
\mathcal{U}^{(s)}
\setminus
\left(\mathcal{T}^{(s)}\cup\mathcal{A}^{(s)}\right).
\end{aligned}
\end{equation}
Accepted tokens remain visible as context for unresolved positions.
Thus, CORA-Diff does not discard accepted tokens.
It only prevents them from being denoised again.

The persistence condition prevents a one-step confidence spike from
triggering acceptance.

\paragraph{Why confidence is informative.}
Let $v=\hat{x}_i^{(s)}$ and
$\gamma_i^{(s)}=\ell_{i,v}^{(s)}-\max_{u\ne v}\ell_{i,u}^{(s)}$
be its current logit margin.
Let $s_{Q,i}^{\star}$ be the pass whose original transfer rule
permanently commits position $i$. This rule commits the current top-1
token, so
$Y_{Q,i}^{\star}=\arg\max_u\ell_{Q,i,u}^{\star}$, where
$\boldsymbol{\ell}_{Q,i}^{\star}$ denotes the pre-transfer logits.
Permanent means no later remask or modification. If the position remains
unresolved through the horizon, these are the final-pass logits and their
deterministically tie-broken top-1 token.

\begin{proposition}[Endpoint-preservation condition]
\label{prop:confidence_margin}
If
\begin{equation}
\label{eq:margin_preservation_condition}
\left\|
\boldsymbol{\ell}_{Q,i}^{\star}
-
\boldsymbol{\ell}_{i}^{(s)}
\right\|_{\infty}
<
\frac{\gamma_i^{(s)}}{2},
\end{equation}
then $\hat{x}_i^{(s)}=Y_{Q,i}^{\star}$.
Moreover, if $c_i^{(s)}>1/2$, then
\(
\gamma_i^{(s)}
\ge
\log\!\bigl(c_i^{(s)}/(1-c_i^{(s)})\bigr).
\)
\end{proposition}

\begin{proof}
Let
$\boldsymbol{\Delta}
=\boldsymbol{\ell}_{Q,i}^{\star}-\boldsymbol{\ell}_{i}^{(s)}$.
For every $u\ne v$,
\begin{align}
\ell_{Q,i,v}^{\star}-\ell_{Q,i,u}^{\star}
&=
\ell_{i,v}^{(s)}-\ell_{i,u}^{(s)}
+\Delta_v-\Delta_u \nonumber\\
&\ge
\gamma_i^{(s)}-|\Delta_v|-|\Delta_u| \nonumber\\
&\ge
\gamma_i^{(s)}-2\|\boldsymbol{\Delta}\|_\infty>0.
\end{align}
Hence $v$ is the unique dense-reference top-1 token and
$v=Y_{Q,i}^{\star}$.
If $u_2$ is the current runner-up, the softmax ratio gives
$\gamma_i^{(s)}=\log(p_{i,v}^{(s)}/p_{i,u_2}^{(s)})$.
Since $p_{i,u_2}^{(s)}\le1-c_i^{(s)}$, the stated lower bound follows.
\end{proof}

Although the reference logits are unavailable online, the condition
formalizes why high confidence creates a wider top-1 stability margin.
Persistence supplies an observable trajectory signal for whether that
margin is likely to be preserved; dense-trace diagnostics and paired
post-intervention analysis test this mechanism empirically.

\subsection{Population Disagreement and Calibration}
\label{sec:theoretical_motivation}

Let $h=(\delta_p,m)$ denote a CORA-Diff configuration. For
prompt $Q$, let $\mathcal{A}_Q(h)$ be the positions accepted by
the residual gate, let $s_i^h$ be the acceptance step for position
$i$, and let $\widetilde{Y}_{Q,i}^{h}$ be its accepted token.
We define prompt-level accepted-token disagreement as
\begin{equation}
\label{eq:accepted_disagreement}
L_Q^{\mathrm{acc}}(h)
=
\frac{1}{\max\{1,|\mathcal{A}_Q(h)|\}}
\sum_{i\in\mathcal{A}_Q(h)}
\mathbb{I}[\widetilde{Y}_{Q,i}^{h}\ne Y_{Q,i}^{\star}].
\end{equation}
We set this loss to zero when no position is accepted.
Its population value is
\begin{equation}
\label{eq:population_accepted_disagreement}
R_{\mathcal{D}}^{\mathrm{acc}}(h)
=
\mathbb{E}_{Q\sim\mathcal{D}}
[L_Q^{\mathrm{acc}}(h)].
\end{equation}
This quantity measures the reliability of the acceptance decisions
themselves.

\paragraph{How persistence controls drift-tail probability.}
Fix $h=(\delta_p,m)$; the trajectory and all state events below depend
on $h$, with superscripts omitted. To preserve the prompt-balanced
weighting in Eq.~\eqref{eq:accepted_disagreement}, let $\nu_Q^h$ be an
offline measure whose first-hit restriction is uniform over accepted
positions within each prompt; the supplement gives its canonical
definition. Probabilities below draw $Q\sim\mathcal D$ and a state from
$\nu_Q^h$. Let $\mathcal R_{i,s}$ mean that
$i$ remains unresolved after the original rule at $s$ and was not
accepted earlier, and let
$\mathcal G_{i,s}=\{c_i^{(s)}\ge\delta_p\}$.
The observable stratum $Z_{i,s}$ contains only confidence, normalized
time, block, and optional dataset/backbone bins. Define
$D_{Q,i}^{(s)}=\|\boldsymbol{\ell}_{Q,i}^{\star}
-\boldsymbol{\ell}_{i}^{(s)}\|_\infty$,
$g(\delta_p)=\frac12\log\frac{\delta_p}{1-\delta_p}$, and
$\mathcal B_{i,s}=\{D_{Q,i}^{(s)}\ge g(\delta_p)\}$.
For $r\ge1$, define
$\mathcal H_{i,s}^{(r)}
=\{\hat x_i^{(s-r+1)}=\hat x_i^{(s-r)}\}$ and
$\mathcal F_{i,s}^{(r)}=\cap_{j=1}^{r}\mathcal H_{i,s}^{(j)}$, with
$\mathcal F_{i,s}^{(0)}=\Omega$. For compactness, set
$\mathcal C_{i,s,z}^{(r)}
=\mathcal F_{i,s}^{(r)}\cap\mathcal R_{i,s}\cap\mathcal G_{i,s}
\cap\{Z_{i,s}=z\}$.

\begin{assumption}[Conditional persistence--drift separation]
\label{ass:persistence_lr}
Let $q_{z,r}^{+}$ and $q_{z,r}^{-}$ be the conditional probabilities
of $\mathcal H_{i,s}^{(r)}$ under $\mathcal B_{i,s}$ and
$\mathcal B_{i,s}^{c}$, respectively, given
$\mathcal C_{i,s,z}^{(r-1)}$.
For every supported $z,r$ with $q_{z,r}^{-}>0$, some
$\lambda_{z,r}>0$ satisfies
\begin{equation}
\label{eq:persistence_lr}
\frac{q_{z,r}^{+}}{q_{z,r}^{-}}
\le e^{-\lambda_{z,r}}.
\end{equation}
\end{assumption}

\begin{proposition}[Persistence-conditioned drift-tail bound]
\label{prop:persistence_drift}
Let
$\pi_z=\Pr(\mathcal B_{i,s}\mid\mathcal C_{i,s,z}^{(0)})$ and
$\Lambda_z(m)=\sum_{r=1}^{m}\lambda_{z,r}$.
Under Assumption~\ref{ass:persistence_lr} and $\delta_p>1/2$,
\begin{equation}
\label{eq:persistence_drift_bound}
\begin{aligned}
\Pr(\hat x_i^{(s)}\ne Y_{Q,i}^{\star}
\mid\mathcal C_{i,s,z}^{(m)})
\;&\le
\Pr(\mathcal B_{i,s}\mid\mathcal C_{i,s,z}^{(m)})
\\[-2pt]
&\le
\frac{\pi_z e^{-\Lambda_z(m)}}
{1-\pi_z+\pi_z e^{-\Lambda_z(m)}}.
\end{aligned}
\end{equation}
\end{proposition}

\begin{proof}
The chain rule under $\mathcal C_{i,s,z}^{(r-1)}$ gives a likelihood
ratio at most
$e^{-\Lambda_z(m)}$, and Bayes' rule gives
Eq.~\eqref{eq:persistence_drift_bound}. Proposition~\ref{prop:confidence_margin}
makes endpoint disagreement a subset of $\mathcal B_{i,s}$.
\end{proof}
Because CORA-Diff accepts a residual candidate as soon as both gates
first hold,
\(\{s_i^h=s,Z_{i,s}=z\}
=\mathcal R_{i,s}\cap\mathcal G_{i,s}\cap
\mathcal F_{i,s}^{(m)}\cap\{Z_{i,s}=z\}
=\mathcal C_{i,s,z}^{(m)}\).
Thus Eq.~\eqref{eq:persistence_drift_bound} bounds disagreement under
the conditional distribution of first-hit states, not pointwise at a
deterministic state. If every stratum with positive first-hit acceptance mass has
$q_{z,r}^{-}>0$, $\pi_z\le\bar\pi$, and
$\lambda_{z,r}\ge\underline\lambda>0$ for all $r\le m$, averaging gives
\begin{equation}
\label{eq:uniform_persistence_consequence}
R_{\mathcal D}^{\mathrm{acc}}(\delta_p,m)
\le
\frac{\bar\pi e^{-m\underline\lambda}}
{1-\bar\pi+\bar\pi e^{-m\underline\lambda}}.
\end{equation}
Because the trajectory, analysis measure, and stratum-specific constants
depend on $h=(\delta_p,m)$, this is a configuration-specific bound;
threshold dependence is measured by the calibration sweep. At
$(0.65,1)$, $g(\delta_p)\approx0.310$ and only $\lambda_{z,1}$ enters.
Supplementary Table~S6 finds marginal prompt-bootstrap upper bounds
below one in five of six strata. The selected operating point therefore
exhibits the predicted contraction in every well-supported stratum; only
the sparsely populated highest-confidence stratum remains statistically
unresolved.

Early acceptance can also change later predictions at unresolved
positions. In the fixed-horizon analysis, let $Y_Q^h$ be the final
CORA-Diff output on the same $n_Q$-token generation canvas as $Y_Q^\star$.
We therefore define prompt-level final-output disagreement as
\begin{equation}
\label{eq:output_disagreement}
L_Q^{\mathrm{out}}(h)
=
\frac{1}{n_Q}
\sum_{i=1}^{n_Q}
\mathbb{I}[Y_{Q,i}^{h}\ne Y_{Q,i}^{\star}],
\end{equation}
with population value
\begin{equation}
\label{eq:population_output_disagreement}
R_{\mathcal{D}}^{\mathrm{out}}(h)
=
\mathbb{E}_{Q\sim\mathcal{D}}
[L_Q^{\mathrm{out}}(h)].
\end{equation}
The accepted-token quantity diagnoses local gate reliability, whereas
the final-output quantity captures downstream changes caused by early
intervention. Together they separate local acceptance fidelity from
trajectory-level effects; benchmark scores evaluate task correctness.

Figure~\ref{fig:persistence_risk} establishes the motivating association
on unperturbed dense trajectories. Supplementary Table~S6 complements it
with paired post-intervention trajectories and directly tests
Assumption~\ref{ass:persistence_lr} on actual CORA-Diff states.

We next analyze the selector used in our experiments. Let
$\mathcal{C}=\{Q_1,\ldots,Q_N\}$ be an independent calibration sample
from $\mathcal{D}$, let $\mathcal{H}$ be a finite threshold grid fixed
before calibration outcomes are examined, and let $h_0$ denote dense
decoding. For prompt $Q$, let $\mathcal K_Q$ be its blocks and let
$S_{Q,b}^{h}$ and $S_{Q,b}^{\mathrm{dense}}$ be the executed and dense
step counts for block $b$. Define the prompt ratio
\(\rho_Q(h):=\sum_bS_{Q,b}^{h}/\sum_bS_{Q,b}^{\mathrm{dense}}\in[0,1]\).
Let $S_Q(h)\in[0,1]$ be its task score. We write their population means as
$M_{\mathcal D}(h)$ and $\rho_{\mathcal D}(h)$, and their calibration
means as $\widehat M_{\mathcal C}(h)$ and
$\widehat\rho_{\mathcal C}(h)$. We similarly define
\begin{equation}
\label{eq:empirical_disagreement}
\widehat R_{\mathcal C}^{r}(h)
=\frac{1}{N}\sum_{j=1}^{N}L_{Q_j}^{r}(h),
\qquad r\in\{\mathrm{acc},\mathrm{out}\}.
\end{equation}
For score tolerance $\tau$ and output-disagreement tolerance
$\epsilon$, the reported empirical rule has the form
\begin{equation}
\label{eq:empirical_selector}
\begin{aligned}
\widehat h&\in
\operatorname*{argmin}_{h\in\mathcal H}\widehat\rho_{\mathcal C}(h),
\\[-2pt]
\text{s.t.}\quad
\widehat M_{\mathcal C}(h)&\ge\widehat M_{\mathcal C}(h_0)-\tau,
\quad
\widehat R_{\mathcal C}^{\mathrm{out}}(h)\le\epsilon .
\end{aligned}
\end{equation}

\begin{proposition}[Uniform bound for the reported selector]
\label{prop:uniform_calibration}
Let $\mathcal H_0=\mathcal H\cup\{h_0\}$ and
\(
\beta_N=\sqrt{\log(8|\mathcal H_0|/\alpha)/(2N)}.
\)
With probability at least $1-\alpha$, the sample means of task score,
prompt step ratio, accepted-token disagreement, and output disagreement
are all within $\beta_N$ of their population means, simultaneously over
$h\in\mathcal H_0$. Consequently, the feasible choice in
Eq.~\eqref{eq:empirical_selector} obeys
\begin{equation}
\label{eq:selected_configuration_bounds}
M_{\mathcal D}(\widehat h)
\ge M_{\mathcal D}(h_0)-\tau-2\beta_N,
\qquad
R_{\mathcal D}^{\mathrm{out}}(\widehat h)
\le\epsilon+\beta_N .
\end{equation}
\end{proposition}
The supplement proves the proposition using Hoeffding's
inequality~\citep{hoeffding1963probability} and a union bound. For
i.i.d.\ calibration and deployment prompts, it quantifies the
population slack induced by data-dependent selection. Benchmark scores
assess correctness, while cross-task transfer is evaluated empirically.
Within each fixed-canvas task and budget, every prompt has the same
dense denominator in this ratio. The prompt mean
therefore equals the aggregate pass-count ratio reported in our
fixed-horizon tables.

\paragraph{From acceptance to runtime.}
When the original rule and CORA-Diff jointly resolve the active block,
the decoder skips its remaining dense passes. We report
$\rho_{\mathrm{step}}$, the executed-step count divided by the fixed
dense horizon; every executed step remains dense. The supplement gives
the derivation, complete algorithm, state, and dense-recovery case.

\section{Experiments}
\label{sec:experiments}

\begin{table*}[t]
\centering
{\small
\setlength{\tabcolsep}{1.7pt}
\renewcommand{\arraystretch}{0.96}
\begin{tabular}{@{}llccccc@{\hspace{4pt}}ccccc@{}}
\toprule
Task & Method
& \multicolumn{5}{c}{256/256}
& \multicolumn{5}{c}{1024/1024} \\
\cmidrule(lr){3-7}\cmidrule(l){8-12}
& & Metric $\uparrow$ & Time (s) $\downarrow$ & Tok/s $\uparrow$ & Spd. $\uparrow$ & Ratio $\downarrow$
& Metric $\uparrow$ & Time (s) $\downarrow$ & Tok/s $\uparrow$ & Spd. $\uparrow$ & Ratio $\downarrow$ \\
\midrule
\multirow{6}{*}{GSM8K}
& Original & $0.7763$ & $38592.97$ & $7.98$ & $1.00\times$ & $1.000$ & $0.7748$ & $252689.01$ & $1.28$ & $1.00\times$ & $1.000$ \\
& Prophet (ICLR'26) & $0.7618$ & $31053.67$ & $9.25$ & $1.24\times$ & $0.7966$ & $0.7789$ & $42660.82$ & $6.97$ & $5.92\times$ & $0.1671$ \\
& KLASS (NeurIPS'25) & $0.7582$ & $20980.90$ & $14.70$ & $1.84\times$ & $0.4309$ & $0.7703$ & $50464.07$ & $6.37$ & $5.01\times$ & $0.1590$ \\
& DAPD (ICML'26) & \best{$0.7801$} & $23239.23$ & $13.25$ & $1.66\times$ & $0.602$ & $0.7824$ & $140928.73$ & $2.31$ & $1.79\times$ & $0.558$ \\
& Learn2PD (ICLR'26) & $0.7680$ & $9461.66$ & $32.55$ & $4.08\times$ & -- & \best{$0.7870$} & $22060.60$ & $14.69$ & $11.45\times$ & -- \\
& \textbf{CORA-Diff (Ours)} & $0.7794$ & \best{$7854.28$} & \best{$39.18$} & \best{$4.91\times$} & \best{$0.2012$} & $0.7862$ & \best{$19231.44$} & \best{$16.96$} & \best{$13.14\times$} & \best{$0.0756$} \\
\midrule
\multirow{6}{*}{MATH}
& Original & \best{$0.2586$} & $101815.16$ & $4.41$ & $1.00\times$ & $1.000$ & \best{$0.2518$} & $681432.76$ & $0.66$ & $1.00\times$ & $1.000$ \\
& Prophet (ICLR'26) & $0.2550$ & $30467.38$ & $14.54$ & $3.34\times$ & $0.2791$ & $0.2475$ & $324491.79$ & $1.39$ & $2.10\times$ & $0.476$ \\
& KLASS (NeurIPS'25) & $0.2552$ & $28451.69$ & $15.40$ & $3.58\times$ & $0.2167$ & $0.2502$ & $216327.86$ & $2.08$ & $3.15\times$ & $0.318$ \\
& DAPD (ICML'26) & $0.2410$ & $58258.52$ & $8.58$ & $1.75\times$ & $0.572$ & $0.2450$ & $342428.52$ & $1.31$ & $1.99\times$ & $0.503$ \\
& Learn2PD (ICLR'26) & $0.2538$ & $12865.90$ & $34.94$ & $7.91\times$ & -- & $0.2490$ & $73541.28$ & $6.12$ & $9.27\times$ & -- \\
& \textbf{CORA-Diff (Ours)} & $0.2546$ & \best{$10728.94$} & \best{$40.45$} & \best{$9.49\times$} & \best{$0.1010$} & $0.2500$ & \best{$63804.57$} & \best{$7.05$} & \best{$10.68\times$} & \best{$0.0868$} \\
\midrule
\multirow{6}{*}{HumanEval} 
& Original & \best{$0.3902$} & $1717.98$ & $8.03$ & $1.00\times$ & $1.000$ & $0.3780$ & $17668.51$ & $1.37$ & $1.00\times$ & $1.000$ \\
& Prophet (ICLR'26) & $0.3537$ & $1248.39$ & $11.11$ & $1.38\times$ & $0.7164$ & $0.3660$ & $3369.55$ & $5.12$ & $5.24\times$ & $0.1879$ \\
& KLASS (NeurIPS'25) & $0.3780$ & $1368.55$ & $10.11$ & $1.26\times$ & $0.6451$ & $0.3780$ & $5896.30$ & $4.12$ & $3.00\times$ & $0.2644$ \\
& DAPD (ICML'26) & \best{$0.3902$} & $848.66$ & $15.52$ & $2.02\times$ & $0.494$ & $0.3841$ & $8335.65$ & $2.50$ & $2.12\times$ & $0.472$ \\
& Learn2PD (ICLR'26) & $0.3598$ & $517.99$ & $26.54$ & $3.32\times$ & -- & \best{$0.3963$} & $2069.93$ & $11.74$ & $8.54\times$ & -- \\
& \textbf{CORA-Diff (Ours)} & $0.3780$ & \best{$395.64$} & \best{$33.97$} & \best{$4.34\times$} & \best{$0.2278$} & $0.3780$ & \best{$1595.87$} & \best{$15.60$} & \best{$11.07\times$} & \best{$0.0894$} \\
\midrule
\multirow{6}{*}{MBPP}
& Original & $0.290$ & $11740.78$ & $7.49$ & $1.00\times$ & $1.000$ & $0.102$ & $81152.95$ & $1.38$ & $1.00\times$ & $1.000$ \\
& Prophet (ICLR'26) & $0.282$ & $6937.01$ & $12.45$ & $1.69\times$ & $0.5776$ & $0.098$ & $10784.85$ & $8.99$ & $7.52\times$ & $0.1312$ \\
& KLASS (NeurIPS'25) & $0.284$ & $5115.24$ & $17.18$ & $2.30\times$ & $0.4116$ & $0.100$ & $14694.81$ & $7.67$ & $5.52\times$ & $0.1701$ \\
& DAPD (ICML'26) & \best{$0.326$} & $5888.99$ & $11.06$ & $1.99\times$ & $0.502$ & $0.102$ & $39514.22$ & $2.72$ & $2.05\times$ & $0.487$ \\
& Learn2PD (ICLR'26) & $0.300$ & $2699.23$ & $32.37$ & $4.35\times$ & -- & $0.102$ & $7400.47$ & $15.72$ & $10.97\times$ & -- \\
& \textbf{CORA-Diff (Ours)} & $0.294$ & \best{$2295.44$} & \best{$38.36$} & \best{$5.11\times$} & \best{$0.1933$} & \best{$0.104$} & \best{$6782.60$} & \best{$18.03$} & \best{$11.96\times$} & \best{$0.0832$} \\
\bottomrule
\end{tabular}
}
\caption{
Main results under the matched Learn2PD-style LLaDA protocol.
Metrics are Flex EM, MATH accuracy, and code pass@1.
CORA-Diff uses the calibration-selected configuration $(\delta_p,m)=(0.65,1)$.
Original denotes the unmodified fixed-horizon dense decoder.
Time is the total wall-clock time over the complete test split; Time and Tok/s are three-run means.
Spd.\ and Ratio denote speedup and denoising-step ratio.
Cache reuse is disabled. Dashes mark unavailable step counts. Raw runs and confidence intervals are in the supplement.
}
\label{tab:main_results_all}
\end{table*}

\begin{table*}[t]
\centering
{\small
\setlength{\tabcolsep}{1.5pt}
\renewcommand{\arraystretch}{0.96}
\begin{tabular}{@{}lcccccccc@{}}
\toprule
Method & \multicolumn{4}{c}{GSM8K} & \multicolumn{4}{c}{HumanEval} \\
\cmidrule(lr){2-5}\cmidrule(l){6-9}
& Metric $\uparrow$ & Len./Blk. $\downarrow$ & Lat. $\downarrow$ & Spd. $\uparrow$
& Metric $\uparrow$ & Len./Blk. $\downarrow$ & Lat. $\downarrow$ & Spd. $\uparrow$ \\
\midrule
Original & $0.7748$ & $1024/32$ & $191.58/214.70$ & $1.00/--$ & $0.3780$ & $1024/32$ & $107.74/152.30$ & $1.00/--$ \\
CORA-Diff & $0.7862$ & $1024/32$ & $14.58/17.61$ & $13.14/--$ & $0.3780$ & $1024/32$ & $9.73/14.55$ & $11.07/--$ \\
Original $+$ EoTP & $0.7810$ & $118.4/4.2$ & $24.12/38.75$ & $7.94/1.00$ & $0.3841$ & $176.3/6.0$ & $24.60/39.80$ & $4.38/1.00$ \\
Learn2PD $+$ EoTP & $0.7865$ & $113.7/4.1$ & $10.36/16.48$ & $18.49/2.33$ & \best{$0.3902$} & $171.5/5.8$ & $9.12/14.90$ & $11.81/2.70$ \\
\textbf{CORA-Diff $+$ EoTP} & \best{$0.7892$} & \best{$112.9/4.0$} & \best{$8.94/13.72$} & \best{$21.43/2.70$} & $0.3841$ & \best{$169.8/5.7$} & \best{$7.41/12.05$} & \best{$14.54/3.32$} \\
\bottomrule
\end{tabular}
}
\caption{
EOS-aware deployment at maximum length and budget 1024.
Paired cells report mean active length/executed blocks (Len./Blk.), mean/P90 latency in seconds (Lat.), and speedup over fixed-horizon Original/Original $+$ EoTP (Spd.).
}
\label{tab:eos_aware_deployment}
\end{table*}

\subsection{Experimental Setup}
\label{sec:experimental_setup}

The main comparison follows the Learn2PD-style LLaDA-8B-Instruct
protocol on four benchmarks~\citep{nie2025llada,bao2025learn2pd,cobbe2021training,hendrycks2021measuring,chen2021evaluating,austin2021program}.
We reproduce all baselines from official implementations under shared
evaluation factors~\citep{li2025prophet,kim2025klass,kim2026dapd}.
All runs use complete test splits, batch size 1, bfloat16,
temperature 0, and one RTX 5090; the supplement gives seeds and prompts.

\paragraph{Calibration and configuration selection.}
We select $(\delta_p,m)$ on 1,000 GSM8K training prompts disjoint from all
reported evaluation sets, using the same 3-shot construction, harness,
decoding configuration, and scoring as the GSM8K test evaluation. The
prespecified rule requires Flex EM within 0.5 percentage points of dense
and final-output disagreement at most 1.0\%, then minimizes step ratio.
Dense Flex EM is 0.780 (floor 0.775); $(0.65,1)$ is selected with 0.779
Flex EM, 0.79\% disagreement, and ratio 0.2020. Only this held-out subset
affects selection; no reported test metric or runtime is used for tuning.
The pair remains fixed in all subsequent experiments.
Proposition~\ref{prop:uniform_calibration} characterizes selection under
the same-distribution assumption; task and backbone transfer are empirical.
Supplementary Table~S1 gives the complete sweep.

\paragraph{Baseline reproduction and controls.}
Each method keeps its released internal rule and recommended thresholds.
Supplementary Table~S2 gives exact settings, official-reference
configurations, and the CreditDecoding exclusion.
Cache reuse and sparse execution are disabled in the main comparison.
CORA-Diff is selected by our calibration rule, whereas reproduced
baselines retain released thresholds or checkpoints. The comparison
therefore provides a controlled evaluation of released operating points
under matched execution factors.

\subsection{Matched-Protocol Results}
\label{sec:main_results}

Table~\ref{tab:main_results_all} compares the standard 256/256 setting with the longer 1024/1024 setting.

\paragraph{Overall quality--efficiency trade-off.}
CORA-Diff has the lowest measured runtime in all eight task--length settings.
Under the fixed-horizon mechanism-isolation protocol, it reaches
$4.34\times$--$9.49\times$ speedup at 256/256 and
$10.68\times$--$13.14\times$ at 1024/1024.
Against Learn2PD, the closest runtime baseline, its measured mean
speedup is 9.1\%--30.9\% higher.

\paragraph{Effect of the denoising horizon.}
CORA-Diff executes 10.10\%--22.78\% of dense steps at 256/256 and
7.56\%--8.94\% at 1024/1024. The lower 1024/1024 ratios are consistent
with a larger late region of stable or redundant denoising work.
Measured speedup reaches 92.7\%--99.5\% of the ideal
$1/\rho_{\mathrm{step}}$ limit. The remaining gap is consistent with
bookkeeping, synchronization, and data-movement overhead.

\paragraph{Task-dependent stopping behavior.}
At 256/256, the step ratio varies from 0.1010 on MATH to 0.2278 on
HumanEval, whereas all four 1024/1024 ratios lie in the narrow
0.0756--0.0894 range. A long horizon therefore creates shared
late-stage redundancy that dominates task differences. At 256/256,
speed depends more on when each block's last unresolved positions
stabilize, rather than on task difficulty alone.

\paragraph{Comparison with baselines.}
CORA-Diff uses fewer reported steps than Prophet, KLASS, and DAPD.
DAPD preserves strong 256/256 scores but accepts conservatively;
Prophet uses answer-level convergence, KLASS tests distributional
stability, and DAPD models token dependencies. At their reported
operating points, these criteria retain more denoising work.
Learn2PD is the closest speed baseline and sometimes scores higher, but
uses a trained filter.
CORA-Diff's measured speedup is 9.1\%--30.9\% higher than Learn2PD's,
and it is the lowest-latency method in all eight matched settings
without a trained acceptance filter.

\paragraph{Executed steps explain the speed gap.}
For methods reporting step counts, ratio closely tracks wall-clock
speed. On GSM8K 1024/1024, CORA-Diff uses a 0.0756 ratio and reaches
$13.14\times$, versus ratios 0.1590--0.558 and speedups
$1.79\times$--$5.92\times$ for Prophet, KLASS, and DAPD.
The values are not exact reciprocals because controller and data-movement
costs differ, but they locate the main gain in earlier block completion.
CORA-Diff tests only residual positions with same-pass signals, converting
local acceptance into skipped dense passes.

\paragraph{Score changes and disagreement.}
Across the eight settings, CORA-Diff's observed metric change from
Original ranges from $-1.22$ to $+1.14$ points, and five scores match or
exceed dense decoding. Selection jointly constrains score and output
disagreement, allowing aggressive step reduction while preserving task
quality. Endpoint agreement measures decoder fidelity, while benchmark
accuracy measures task utility; using both prevents speed from being
optimized at the expense of answer-critical tokens.

\paragraph{Scope of the fixed-horizon comparison.}
The table controls shared factors; official settings are in the
supplement. The fixed 1024-token canvas is a mechanism-isolation stress
test of repeated denoising. We next remove post-response work with a
shared EOS-aware controller to measure deployment benefit.

\subsection{Deployment and Robustness Analysis}
\label{sec:deployment_robustness}

\paragraph{EOS-aware deployment.}
At maximum length and budget 1024, only a committed EOS removes the
suffix and future blocks. This removes post-response work, while
CORA-Diff reduces denoising before EOS. Table~\ref{tab:eos_aware_deployment}
shows $2.70\times$ and $3.32\times$ speedups over EOS-aware dense
decoding. Mean and P90 latency beat Learn2PD, although HumanEval scores
lower; similar lengths make shorter outputs an unlikely main cause.
The result is therefore consistent with CORA-Diff removing substantial
pre-EOS denoising rather than relying mainly on earlier truncation.

\paragraph{Cross-backbone transfer and cache compatibility.}
Supplementary Table~S3 reports $3.18\times$--$3.53\times$ on Dream
without retuning and cache combinations up to $53.65\times$
~\citep{ye2025dream,bao2025learn2pd,wu2025fastdllm}. These results
demonstrate transfer to a second backbone and composability with cache
reuse.

\paragraph{Component ablation.}
Table~\ref{tab:core_ablation} reports independent complete-test-set runs,
separate from Table~S1 and unused for selection. At frozen
$\delta_p=0.65$, one persistence check improves confidence-only Flex EM
from 0.7301 to 0.7794, while Tok/s falls from 43.51 to 39.12 (10.1\%) and
ratio rises from 0.1783 to 0.2012. Persistence alone gives 0.7557 Flex EM
at 35.06 Tok/s. Two checks reach ratio 0.2361 and 34.47 Tok/s without
improving over $m=1$, supporting the combined one-check rule.

\begin{table}[ht]
\centering
{\small
\setlength{\tabcolsep}{5.0pt}
\renewcommand{\arraystretch}{0.94}
\begin{tabular}{@{}lccc@{}}
\toprule
Variant & Flex EM $\uparrow$ & Tok/s $\uparrow$ & Ratio $\downarrow$ \\
\midrule
Confidence only ($m=0$) & $0.7301$ & \best{$43.51$} & \best{$0.1783$} \\
Persistence only & $0.7557$ & $35.06$ & $0.2245$ \\
Both ($m=1$) & \best{$0.7794$} & $39.12$ & $0.2012$ \\
Both ($m=2$) & $0.7763$ & $34.47$ & $0.2361$ \\
\bottomrule
\end{tabular}
}
\caption{Independent post-selection component ablation on the complete GSM8K test split ($256/256$, $\delta_p=0.65$); not used for selection. Tok/s averages three runs; Flex EM and Ratio are deterministic at temperature 0.}
\label{tab:core_ablation}
\end{table}

\section{Conclusion, Limitations, and Future Work}
\label{sec:conclusion}

CORA-Diff retains the original transfer rule and terminates resolved
blocks without a learned acceptance model. It is fastest in all eight
matched settings, gives $2.70\times$--$3.32\times$ EOS-aware
speedups, and reaches $13.14\times$ in fixed-horizon mechanism isolation.
Theory and paired trajectories connect confidence and persistence to
endpoint agreement, explaining the observed reduction in dense computation.

\paragraph{Limitations and future work.}
Paired validation supports five of six strata; one sparse stratum remains
data-limited, and broader deployment tests are future work.

\clearpage
\bibliography{references}
\end{document}